\documentclass{article}

 \PassOptionsToPackage{numbers, compress}{natbib}
 \usepackage[final]{nips_2017}

\usepackage[utf8]{inputenc} % allow utf-8 input
\usepackage[T1]{fontenc}    % use 8-bit T1 fonts
\usepackage[colorlinks,citecolor=blue,urlcolor=blue]{hyperref}       % hyperlinks
\usepackage{url}            % simple URL typesetting
\usepackage{booktabs}       % professional-quality tables
\usepackage{amsfonts,amsmath,amssymb}       % blackboard math symbols
\usepackage{nicefrac}       % compact symbols for 1/2, etc.
\usepackage{microtype}      % microtypography

\usepackage{color,graphicx}
\usepackage{algpseudocode,algorithm,algorithmicx}
\usepackage{enumitem}% http://ctan.org/pkg/enumitem
\usepackage{subcaption}

\usepackage[normalem]{ulem}
\newtheorem{theorem}{Theorem}
\newtheorem{lemma}[theorem]{Lemma}

\newtheorem{definition}[theorem]{Definition}

\newenvironment{proof}{{\bf Proof.}}{$\Box$}

\newcommand{\N}{\mbox{$\mathbb{N}$}}

\title{A Note on Community Trees in Networks}

\author{
  Ruqian Chen\thanks{Department of Mathematics}\\
  University of Washington\\
  Seattle, WA 98105 \\
  \texttt{ruqian@uw.edu} \\
  \And
  Yen-Chi Chen\thanks{Department of Statistics} \\
  University of Washington\\
  Seattle, WA 98105 \\
   \texttt{yenchic@uw.edu} \\
  \AND
  Wei Guo\thanks{Department of Industrial \& Systems Engineering}\\
  University of Washington\\
  Seattle, WA 98195 \\
   \texttt{weig@uw.edu} \\
   \And
  Ashis G. Banerjee\thanks{Department of Industrial \& Systems Engineering and Department of Mechanical Engineering}\\
  University of Washington\\
  Seattle, WA 98195 \\
  \texttt{ashisb@uw.edu} \\
}

\begin{document}
% \nipsfinalcopy is no longer used

\maketitle

\begin{abstract}
We introduce the concept of community trees that summarizes topological structures within a network. 
A community tree is a tree structure representing clique communities from the clique percolation method (CPM). 
The community tree also generates a persistent diagram.
Community trees and persistent diagrams reveal topological structures of the underlying networks and can be used as visualization tools. 
We study the stability of community trees and derive a quantity called the total star number (TSN) that presents an upper bound on the change of community trees. 
Our findings provide a topological interpretation for the stability of communities generated by the CPM.
\end{abstract}

\section{Introduction}

The clique percolation method (CPM \cite{understanding_Palla_2005}) is
a well-known and powerful algorithm for community detection \cite{porter2009communities,fortunato2010community}
in networks
that has been applied to various complex networks such as
social networks \cite{understanding_Palla_2005, palla2007quantifying}, biological networks \cite{adamcsek2006cfinder,zhang2006identification},
and collaboration networks \citep{pollner2005preferential,palla2007quantifying,duan2012incremental}.
Given an order of cliques, say $k$, the CPM detects communities using overlaps between the $k$-cliques.
Since the CPM uses cliques to construct the communities,
one vertex may be assigned to multiple communities, and this fact can be used to describe the overlaps among the communities \cite{understanding_Palla_2005,porter2009communities}. 
The communities found by the CPM are, thus, called clique communities. 

In this paper, we introduce the concept of a community tree of a network. 
Here, we consider the simplest case where a network is an undirected and unweighted graph. 
A community tree is a tree structure representing the clique communities discovered by the CPM. 
A key characteristic of a community tree is:
instead of using a fixed order of the cliques $k$, 
the construction of a community tree is based on the clique communities for all possible value of $k$. 
It uses the fact that the collection of all possible clique communities (regardless of the order of the cliques)
forms a tree structure. 

%The community tree is a tree structure representing
%clique communities discovered by the clique percolation method (CPM \cite{understanding_Palla_2005}),

A community tree is generalized from the notion of the cluster tree 
\cite{stuetzle2003estimating,chaudhuri2010rates,balakrishnan2013cluster,chaudhuri2014consistent,chen2016generalized,jisu2016statistical} 
of a function
in topological data analysis (TDA \cite{wasserman2016topological}).
Similar to the fact that a cluster tree summarizes the creation and elimination of connected components
of a function, a community tree summarizes the creation and elimination of communities
of a network. 
A community tree naturally generates 
a persistent diagram \cite{cohen2007stability}, a popular analytical tool in TDA.
Using a community tree or a persistent diagram,
we are able to analyze topological structures
of a network using the notion of communities. 
Thus, this paper provides a new direction to
bridge TDA and network science.

%{\em Main Contributions.}

\section{Background}

\subsection{Clique community and clique percolation method}

Let $G$ be an unweighted, undirected graph (network) with a vertex set $V(G)$ and an edge set $E(G)$. 
A \textit{clique} refers to a graph $C \subseteq G$ such that any two distinct vertices $u,v \in V(C)$ are adjacent (i.e. share an edge). 
A $N$-clique is a clique with $N$ vertices. 
Assume there exists $C_1,\cdots,C_n$, $n$ different $k$-cliques within $G$. 
Given an integer level $k$ where $k\geq 2$,
the CPM \cite{understanding_Palla_2005} is an algorithm that 
finds every $k$-clique within the network $G$ (i.e., $C_1,\cdots,C_n$)
and partitions these $k$-cliques based on an adjacency matrix $A\in\{0,1\}^{n\times n}$
such that $A_{ij} = 1$ if $C_i\cap C_j$ contains $k-1$ vertices and $0$ otherwise. 
Namely, CPM separates $k$-cliques into connected components using the adjacency matrix. 
The subgraph generated by the union of $k$-cliques within the same connected component 
is called a {\bf$k$-clique community} ({\bf$k$-community} for short) \cite{understanding_Palla_2005}. We define a $1$-community to be any graph. 
%Namely, a $k$-community is a subgraph
%that is generated/spanned by several $k$-cliques.

If a $k$-community $\mathcal{C} = \cup_{\ell=1}^L C_\ell$ is created by $C_1,\cdots,C_L$, 
$L$ $k$-cliques, 
then for any two $k$-cliques $C_i$ and $C_j$,
there exists a sequence of $k$-cliques $C_{\ell_1}=C_i,C_{\ell_2},\cdots, C_{\ell_m} = C_j$
such that 
$C_{\ell_a}\cap C_{\ell_{a+1}} $ is a $(k-1)$-clique. We call such a sequence a \textbf{$k$-clique path} 
$P$.
%with order $||P||=m$. Note that $P$ is a subgraph of $G$, but not necessarily an induced subgraph. 
%\sout{`clique path' $C_i=C_{\ell_1},C_{\ell_2},\cdots, C_{\ell_m} = C_j$ such that  $C_{\ell_a}\cap C_{\ell_{a+1}} $ is a $(k-1)$-clique.}
If a community is formed by $k$-cliques, we say that this community has an {\bf order} $k$.

%and separate these $k$-cliques
%into groups such that
%for any two $k$-cliques $C_1,C_2$ within the same group,
%there exist a sequence of $k$-cliques $C_{1} = C_{}$ within that group 

%Given a $k$-clique $C$ within $G$, 
%the clique percolation method (CPM) finds all other $k$-cliques 
%that share $(k-1)$ vertices with $C$.
%By applying the CPM 

%Based on the definition of \cite{understanding_Palla_2005}, 
%a {\bf$k$-clique community} ({\bf$k$-community} for short) is the union of $k$-cliques that can be reached by the clique percolation method (CPM).

\subsection{Community trees}

An important feature of a $k$-community is its nested structure. 
Since any $k$-clique contains $k$ distinct $(k-1)$-cliques,
any $k$-community is a subset (subgraph) of a $(k-1)$-community. 
Thus, 
for any $k$-community $\mathcal{C}_k$, 
there exists a sequences of communities (at different orders)
$\mathcal{C}_{k-1},\mathcal{C}_{k-2},\cdots,\mathcal{C}_{1}$
such that each $\mathcal{C}_{\omega}$ is a $\omega$-community and
$$
\mathcal{C}_k\subset \mathcal{C}_{k-1}\subset\cdots\subset\mathcal{C}_1.
$$
Note that the $1$-community is the original graph $G$.

The nested structure of clique communities defines a tree structure for the collection of
all the communities across various orders. 
Given any integer $k\in\N$, let 
$\mathcal{C}_{k,1},\cdots,\mathcal{C}_{k,J(k)}$ be the $k$-communities of $G$
and 
$$
\mathbb{C}_k = \{\mathcal{C}_{k,j}:j=1,\cdots,J(k)\}
$$
be the collection of them. 
Then the collection of all communities
$$
\mathbb{C} = \bigcup_{k\in\N} \mathbb{C}_k
$$
has a tree structure. 
%as indicated by Theorem~\ref{thm::tr}:
%\begin{theorem}
%The $\mathbb{C}$ constructed from CPM forms a tree. That is, it satisfies the tree axiom: for any $\mathcal{C}_1$ and $\mathcal{C}_2$ in $\mathbb{C}$, either $\mathcal{C}_1\subseteq  \mathcal{C}_2$ or $\mathcal{C}_2\subseteq \mathcal{C}_1$ or $\mathcal{C}_1\cap \mathcal{C}_2=\emptyset$.
%\label{thm::tr}
%\end{theorem}
%\begin{proof}
%Let $\mathcal{C}_1\in \mathbb{C}_{k_1}$ and $\mathcal{C}_2\in\mathbb{C}_{k_2}$ such that $\mathcal{C}_1$ and $\mathcal{C}_2$ are different nodes in the community tree. Without loss of generality, let $k_1\leq k_2$. 
%\begin{itemize}
%\item Case 1: $k_1=k_2$. Because $\mathcal{C}_1$ and $\mathcal{C}_2$ are different nodes, they represent different communities at level $k_1$. They are different equivalent classes of cliques and therefore don't intersect at any clique of size $\geq k-1$.
%\item Case 2: $k_1< k_2$. Assume $\mathcal{C}_1\cap \mathcal{C}_2\neq\emptyset$. For any two cliques $V$ and $W$ in $\mathcal{C}_2$ such that $V$ and $W$ are in the same community, $V$ and $W$ share at least $k_2-1$ many vertices, that is, $|V\cap W|
%\geq k_2-1$. Because $k_1\leq k_2$, $|V\cap W|\geq k_1-1$. Hence $V$ and $W$ are in the same community in $\mathcal{C}_1$. Therefore $\mathcal{C}_2\subseteq \mathcal{C}_1$. 
%\end{itemize}
%\end{proof}
We call the tree generated by $\mathbb{C}$ the {\bf community tree} of $G$.
Figure~\ref{fig::tr1} displays an example of a community tree and 
its corresponding communities. 
%Note that the community tree can be viewed as a generalization of the cluster tree of a function \cite{chen2016generalized,jisu2016statistical}. 

\begin{figure}
\center
\includegraphics[height=4in]{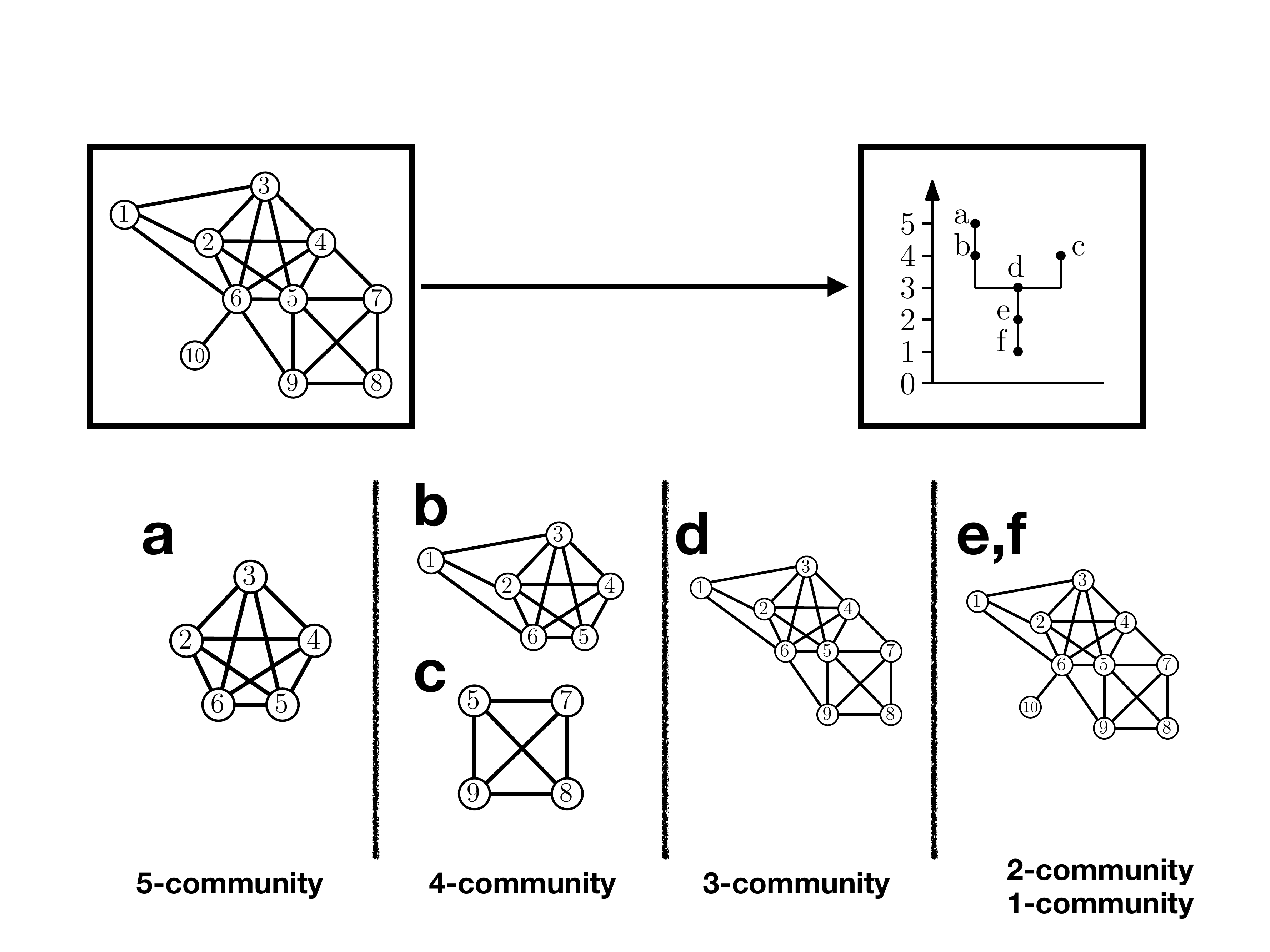}
\caption{The construction of a community tree from a network.
The top left panel shows the original network (graph). 
The top right panel shows the resulting community tree.
To construct the community tree, we consider $k$-community at various levels.
The panels in the bottom row present
all the communities at various orders.
The alphabetical letter denotes the corresponding position of each community in the cluster tree (top right).
}
\label{fig::tr1}
\end{figure}

%The community tree has a powerful feature:
Every node of a community tree
corresponds to a unique community
so a community tree informs how different communities are associated to each other. 
Starting at a leaf of a community tree, 
when we reduce the order of communities (moving down the tree toward the root), we see
how one community morphs into another community.
Moreover, communities from different branches of the tree may merge when we follow them down to the root.
The order that two communities merge tell us how these two communities overlap. 
Using a community tree, we can visualize a complex network and its communities.

\section{Stability of community trees}

\subsection{Persistent diagram and tree metrics}

To summarize the shape of a community tree,
we introduce the concepts of {\bf components} and {\bf persistent diagram (PD)} \cite{cohen2007stability}. 
A component of a community tree is a nested sequence of
communities that starts with
a community at a leaf of the tree, say a $T$-community $\mathcal{C}_{T}$,
and then followed by $\mathcal{C}_{T-1},\mathcal{C}_{T-2},\cdots,\mathcal{C}_{1}$
such that
$$
\mathcal{C}_T\subset \mathcal{C}_{T-1}\subset\cdots\subset \mathcal{C}_1,
$$
where $\mathcal{C}_{\omega}$ is a $\omega$-community. 
Every leaf in a community tree corresponds to
a unique component and the entire tree can be reconstructed using all these components. 
Figure~\ref{fig::tr2} shows the components in the tree of Figure~\ref{fig::tr1}.

For each component, we define its {\bf birth time} (birth order or birth level) to be the highest order of its communities (this corresponds to the order of its leaf community).
Two components will merge at certain order/level.
Whenever two components merge, we compare their birth time.
The one with a higher birth time will stay alive
and the other one will be eliminated. 
And we define the order of this merging to be
the {\bf death time} of the younger component. 
Note that if two components are at the same age, we arbitrarily pick one of them to stay alive. 
Moreover, for the component that cannot be assigned a death time using this way, we will set its death time to be $1$. 

The difference between birth time and death time of a component is called the {\bf persistence or life time}.
A component with a higher persistence is often related to the communities that are more robust against change in the network.
%Having defined the birth time and death time for every component of the community tree, we can then construct a persistent diagram of it.
Assume that a community tree has $M$ components
and each have birth and death time $(d_1,b_1),\cdots,(d_M,b_M)$. 
A PD of a community tree is the collection of
birth time and death time for each component along with the line $x=y$. 
Namely, 
$$
PD = \{(d_i,b_i): i=1,\cdots,M\}\cup \{(d,b): d=b\}.
$$

\begin{figure}
\center
\includegraphics[height=2.4in]{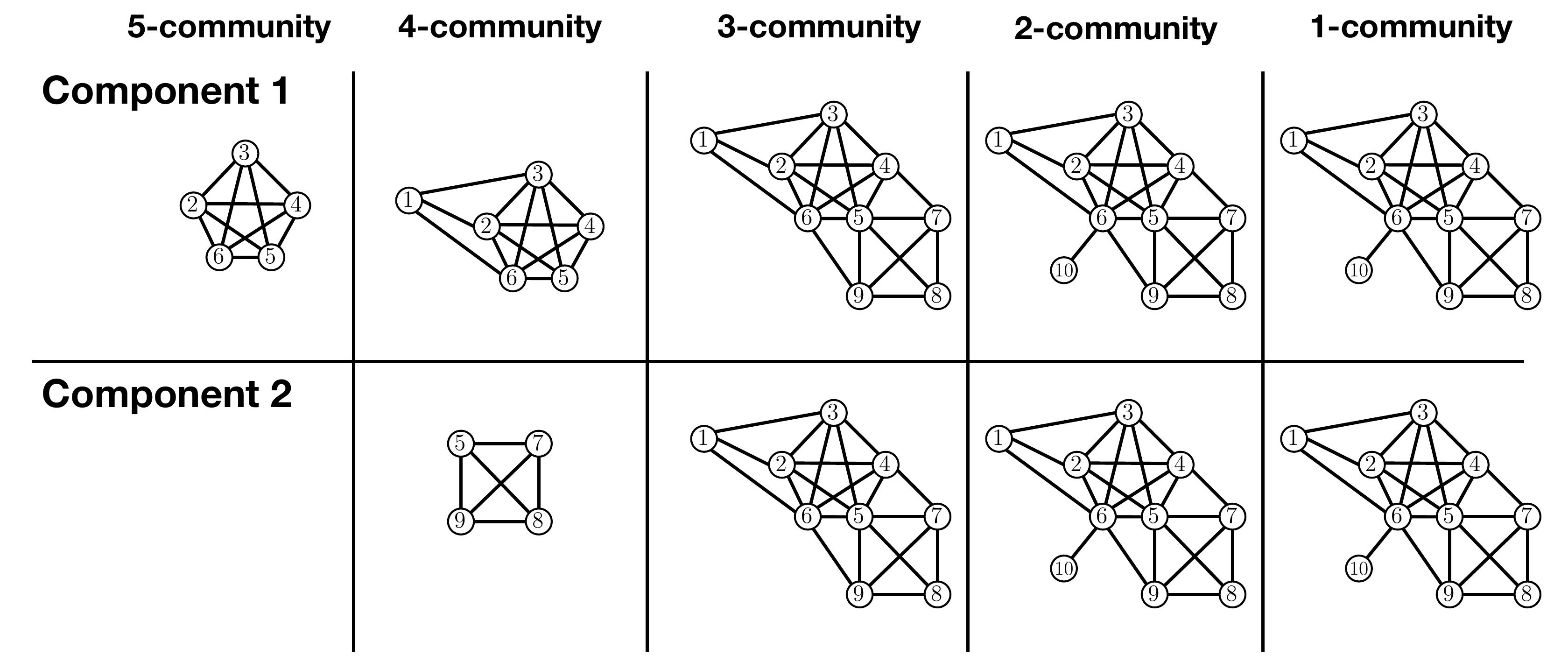}
\caption{Components of the community tree in Figure~\ref{fig::tr1}.
There are two components in the community tree since the tree has two leafs. 
The component 1 starts with the $5$-community (the label $a$ community in Figure \ref{fig::tr1}).
The component 2 is the one starts with the $4$-community (the label $c$ community in Figure \ref{fig::tr1}).
We see that the two components merge at the order of $3$, i.e., their $3$-communities (and any community of a smaller order) are the same.
Since component 2 has birth time 4, which is lower than component 1 (birth time 5), it is eliminated due to this merging and its
death time is $3$ (the order of this merging).
Note that component 1 never merged into others so its death time is $1$.}
\label{fig::tr2}
\end{figure}

\begin{figure}
\center
\includegraphics[height=1.8in]{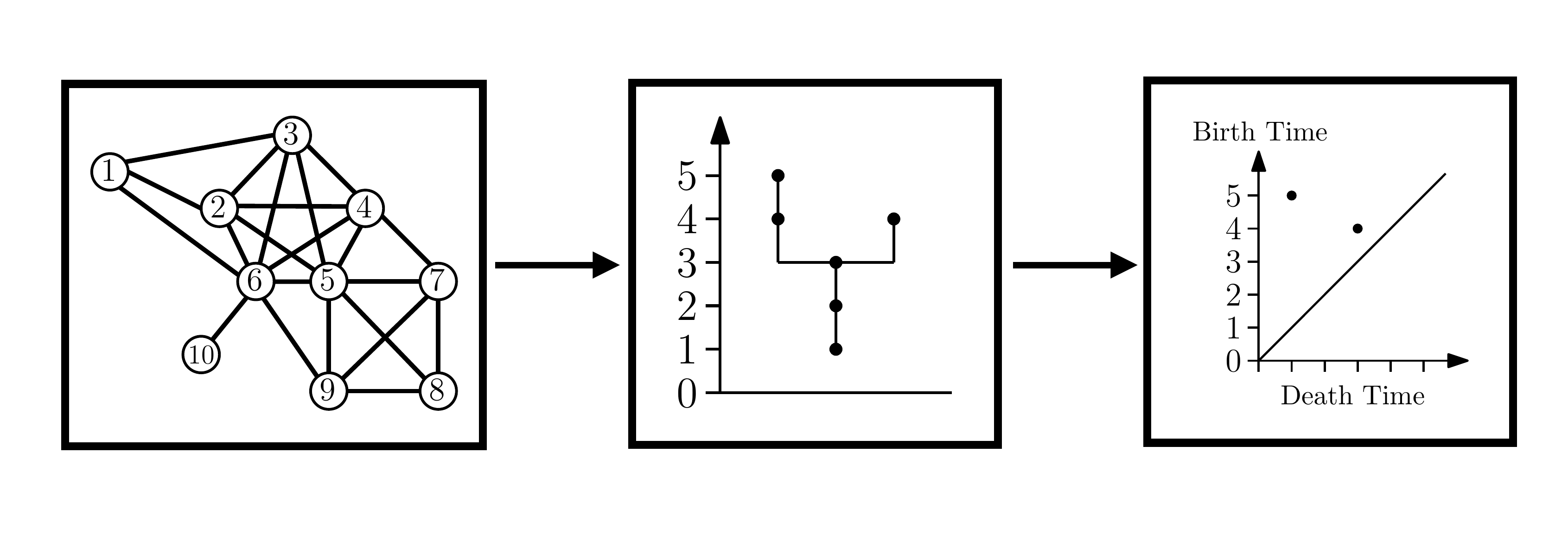}
\caption{Original network, community tree,
and the persistence diagram correspond to Figure~\ref{fig::tr1}. 
By examining the components described in Figure~\ref{fig::tr2}, we see that there are two components. Component 1 has birth time $b_1=5$ and death time $d_1=1$
and component 2 has a birth time of $b_2=4$ and death time $d_2=3$ (the order of merging). 
The persistent diagram (right panel) is then the two points $(1,5), (3,4)$ and the line $x=y$.}
\label{fig::tr3}
\end{figure}

The persistent diagram is a 2D diagram representing topological structures of a community tree. 
Every community tree admits a unique persistent diagram. 
The elements in the persistent diagram represents the robustness of each component in the community tree. 
Figure~\ref{fig::tr3} shows the construction of a persistent diagram using the network presented
in Figure~\ref{fig::tr1}.

Persistent diagrams provide a way to measure the change in community trees. 
%This is particularly useful because changes in community trees are difficult to quantify but there are well-studied metrics for persistent diagrams. 
We define the changes of community tree in terms of the changes of the corresponding persistent diagrams. 
In particular, we use the bottleneck distance \cite{cohen2007stability} of persistent diagrams to measure the change in community trees. 
For
any two persistent diagrams $PD_1$ and $PD_2$,
let $\gamma:PD_1\mapsto PD_2$ be a bijective (one-to-one and onto) mapping between them. 
The {\bf bottleneck distance} is
$$
d_\infty(PD_1,PD_2) = \inf_\gamma \sup_{A\in PD_1} \|A-\gamma(A)\|_\infty,
$$
where the infimum is taken over all possible bijective mappings. And for a vector $v=(v_1,v_2)$, the norm $\|v\|_\infty = \max\{|v_1|,|v_2|\}$ is the $L_\infty$ norm.

The bottleneck distance has an important relation with  components: if a component has persistence $L$, we need a change with at least a bottleneck distance of $L$ to eliminate this component.

Let $T_1$ and $T_2$ denote the two community trees and $PD(T_1), PD(T_2)$ be the corresponding persistent diagrams.
We then define the bottleneck distance between the two community trees as
$$
d_B(T_1,T_2) = d_\infty(PD(T_1),PD(T_2)).
$$
The bottleneck distance measures how the community trees differ in terms of their corresponding persistent diagrams.
Figure~\ref{fig::tr4} provides examples
of computing bottleneck distance of community trees.

Note that there are other metrics for trees constructed from a function 
\cite{balakrishnan2013cluster,chaudhuri2010rates,chaudhuri2014consistent,eldridge2015beyond,jisu2016statistical}. 
However, these metrics cannot be applied to community trees
because community trees are not constructed from a function.

\begin{figure}
\center
\includegraphics[height=3.5in]{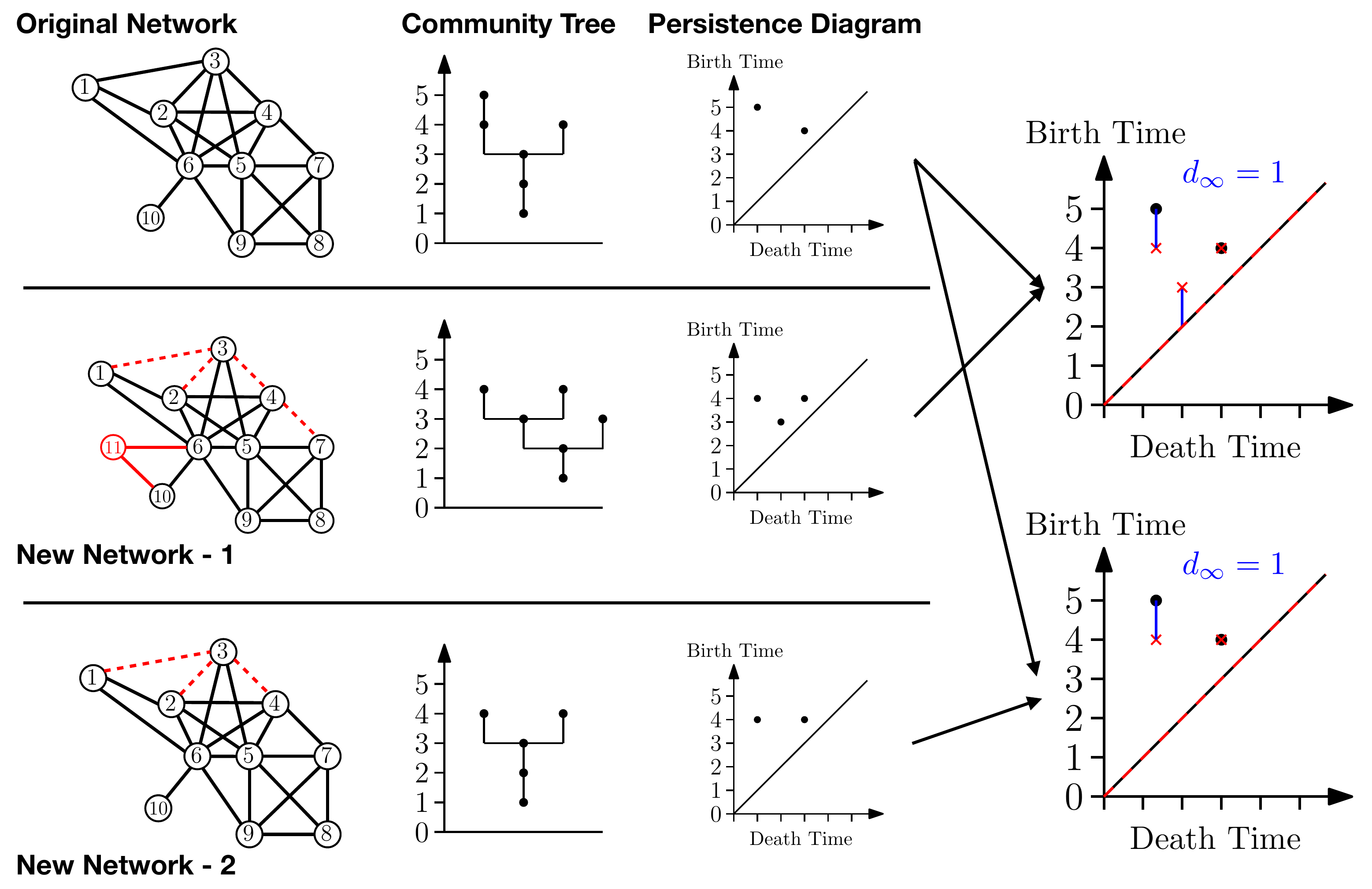}
\caption{The difference between two networks.
We compare one network (top row) to
two possible new networks (middle and bottom rows).
%The top row presents the original network
%and the middle and bottom row are two new networks.
The first column displays how the networks look like.
For the new networks (middle and bottom),
the dashed red edges are edges being removed and the
solid red edges/vertex are the new edges/vertex being added. 
The second column presents the community trees.
The third column shows the persistent diagrams. 
The right panel compares persistent diagrams (black: original network; red: new network). 
The blue line indicates an optimal matching between two persistent diagrams that leads to the bottleneck distance  between them.
When comparing the first new network (new network -1) to the original network, 
the $RSN$ concerns the dashed edges (edges being removed), so it equals $2$ (the set $V_0$ can be chosen as $V_0 =\{v_3,v_4\}$). 
The $ASN$ is $1$ since the all the added edges are incident to vertex $v_{11}$.
Thus, $TSN = 2+1 = 3$ gives a conservative bound. 
In the case of comparing the original network to the second new network (new network - 2), the $RSN$ is 1 and $ASN$ is 0, leading to $TSN=1$, which agrees with the actual bottleneck distance.
}
\label{fig::tr4}
\end{figure}

\subsection{Stability theory}

%Given a network $G$, let $T(G)$ denotes its community tree. 
%We will present a stability theory for the community trees. 
%
%
For any edge $e$, let $\nu(e)$ be the two
vertices of $e$.
%For any graph $G$, let $E(G)$ denotes the collection
%of its edges.
For two graphs $G$ and $G'$,
we introduce a quantity called star number 
that will be useful in deriving the stability theory. 

\begin{definition} 
The \textbf{removal star number (RSN)} of $G'$ and $G$
is
$$
RSN(G',G) = \min\{|V_0|: \nu(e)\cap V_0\neq\emptyset\,\,\, \forall e\in E(G)\backslash E(G')\},
$$
where $V_0$ is a collection of vertices
and $|V_0|$ is the number of elements in the set $V_0$.
The \textbf{addition star number (ASN)} of $G'$ and $G$
is
$$
ASN(G',G) = \min\{|V_0|: \nu(e)\cap V_0\neq\emptyset\,\,\, \forall e\in E(G')\backslash E(G)\}
$$
The \textbf{total star number (TSN)} is the sum of RSN and ASN.
\label{def::sn}
\end{definition}
If $TSN(G,G')=k$,
then we can interpret it as that 
the change from $G$ to $G'$
can be attributed to about $k$ vertices. 

\begin{theorem}
Let $G$ and $G'$ be two graphs.
Then their community trees satisfy
$$
d_B(T(G),T(G')) \leq TSN(G',G). 
$$
\label{thm::star}
\end{theorem}

The proof of Theorem~\ref{thm::star} is long so we defer it to the appendix. 
Theorem~\ref{thm::star} provides a powerful result: 
the difference between two community trees is bounded by
their $TSN$. 
This implies that it is possible that
the community tree does not change much
even the network has been substantially changed. 
For instance, if a network has a vertex of high degree, 
removing all edges attached to this vertex 
only contributes to a $TSN = 1$ effect on the community tree.
So the community tree may remain unchanged or only slightly changed. 
In a sense, the $TSN$ describes an upper bound to the effective change to the community tree. 
Figure~\ref{fig::tr4} provides the $TSN$ for comparing different networks. 
We also include the actual bottleneck distance as a reference.

%only related to about $k$ vertices.

$TSN$ can be computed without constructing community trees and persistent diagrams. 
Thus, to roughly compare the community tree difference between two networks, 
we do not need to actually build the community trees and persistent diagrams but just compute their $TSN$.

However, computing the $TSN$ could be very difficult as described in the following theorem.

\begin{theorem}
Let $G$ and $G'$ be two graphs.
Computing the $RSN(G',G)$ or $ASN(G',G)$
is an NP-complete problem.
\label{thm::NP}
\end{theorem}
\begin{proof}
We only prove the case of $RSN$ since the case of $ASN$ is just swap the role of $G$ and $G'$. 
Let $V(G)$ denotes the vertex in $G$ and $E^\Delta = E(G)\backslash E(G')$. 
We define a new graph $G^\Delta = (V(G),E^\Delta)$. 
This graph is the graph where the edges representing 
the edge difference between $G$ and $G'$. 

Thus, the number $RSN(G',G)$ is to find
the minimum number of vertices in $V(G)$
such that every edge in $E^\Delta = E(G)\backslash E(G')$
is incident to at least one element in the subset of vertices. 
Namely, $RSN(G',G)$ is equivalent to the size of minimum vertex cover of the graph $G^\Delta$.
Since finding the minimum vertex cover is an NP-complete problem \cite{yannakakis1980edge,dinur2005hardness}, 
computing $RSN(G',G)$ is also NP-complete. 

\end{proof}

\section{Examples}

\begin{figure}
\centering
\begin{subfigure}[b]{0.3\textwidth}
        \centering
        \includegraphics[height=1.3in]{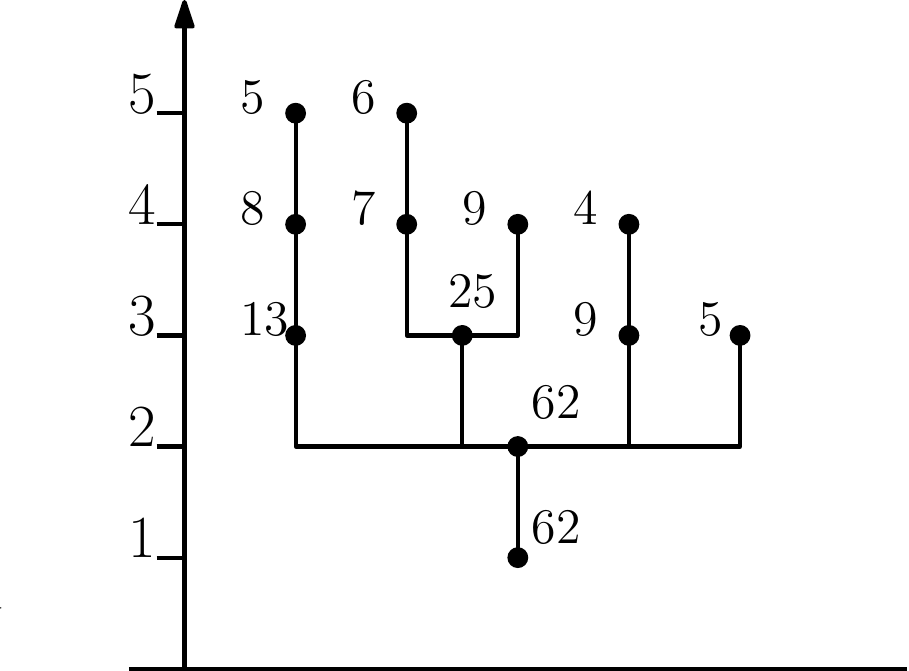}
		\includegraphics[height=1.3in]{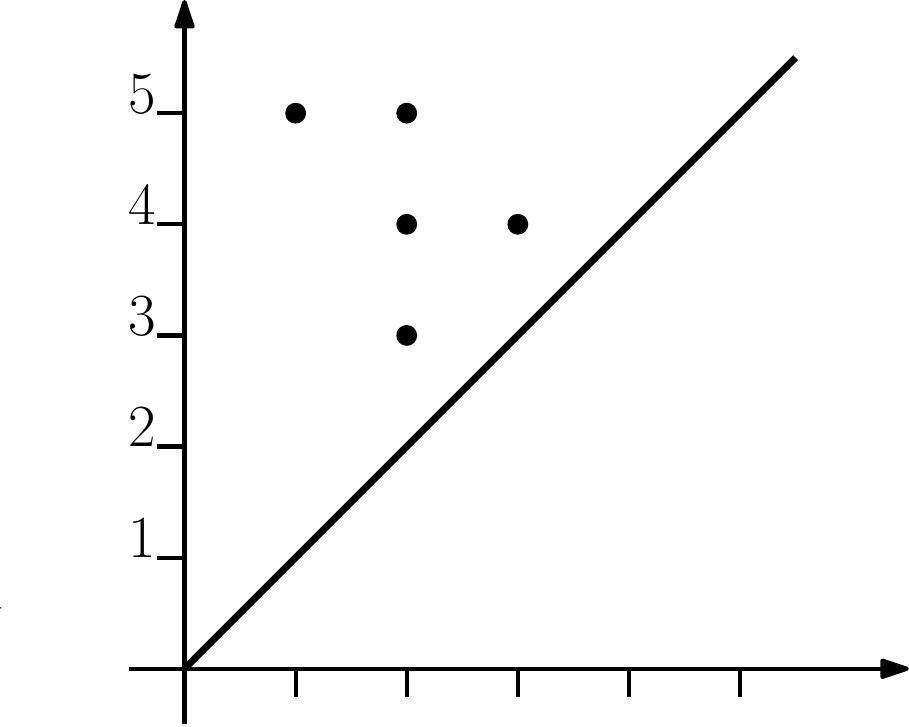}
        \caption{Dolphin network}
    \end{subfigure}
\begin{subfigure}[b]{0.3\textwidth}
        \centering
        \includegraphics[height=1.3in]{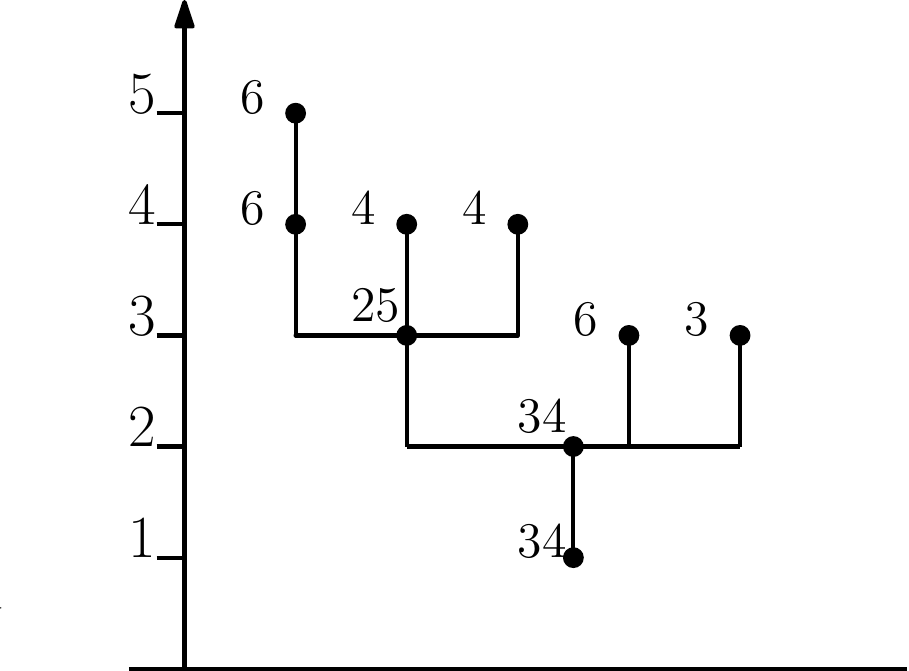}
		\includegraphics[height=1.3in]{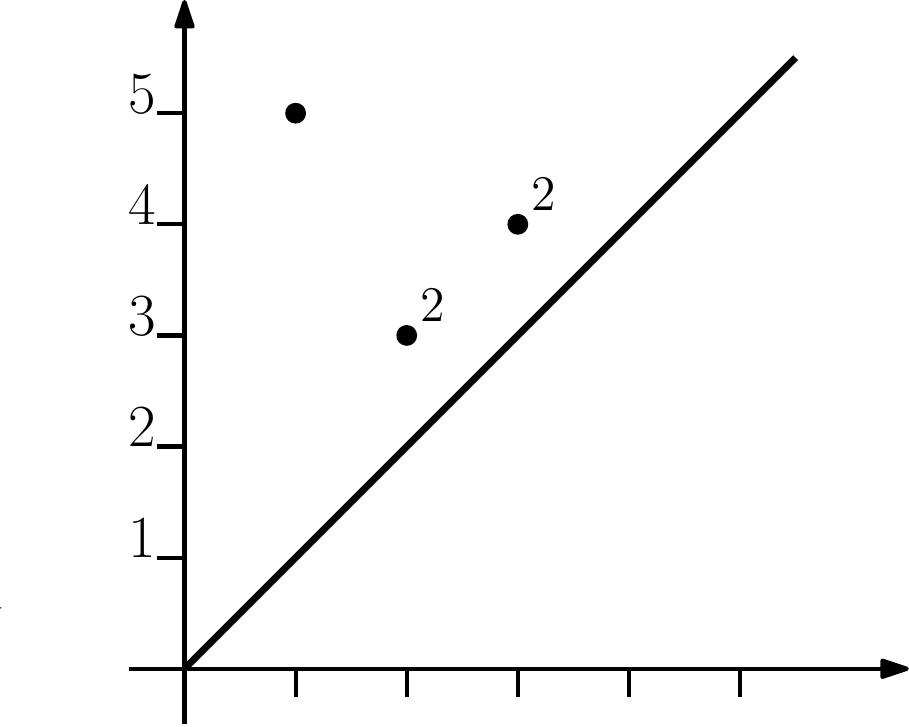}
        \caption{Zachary's karate club}
    \end{subfigure}
\begin{subfigure}[b]{0.3\textwidth}
        \centering
        \includegraphics[height=1.3in]{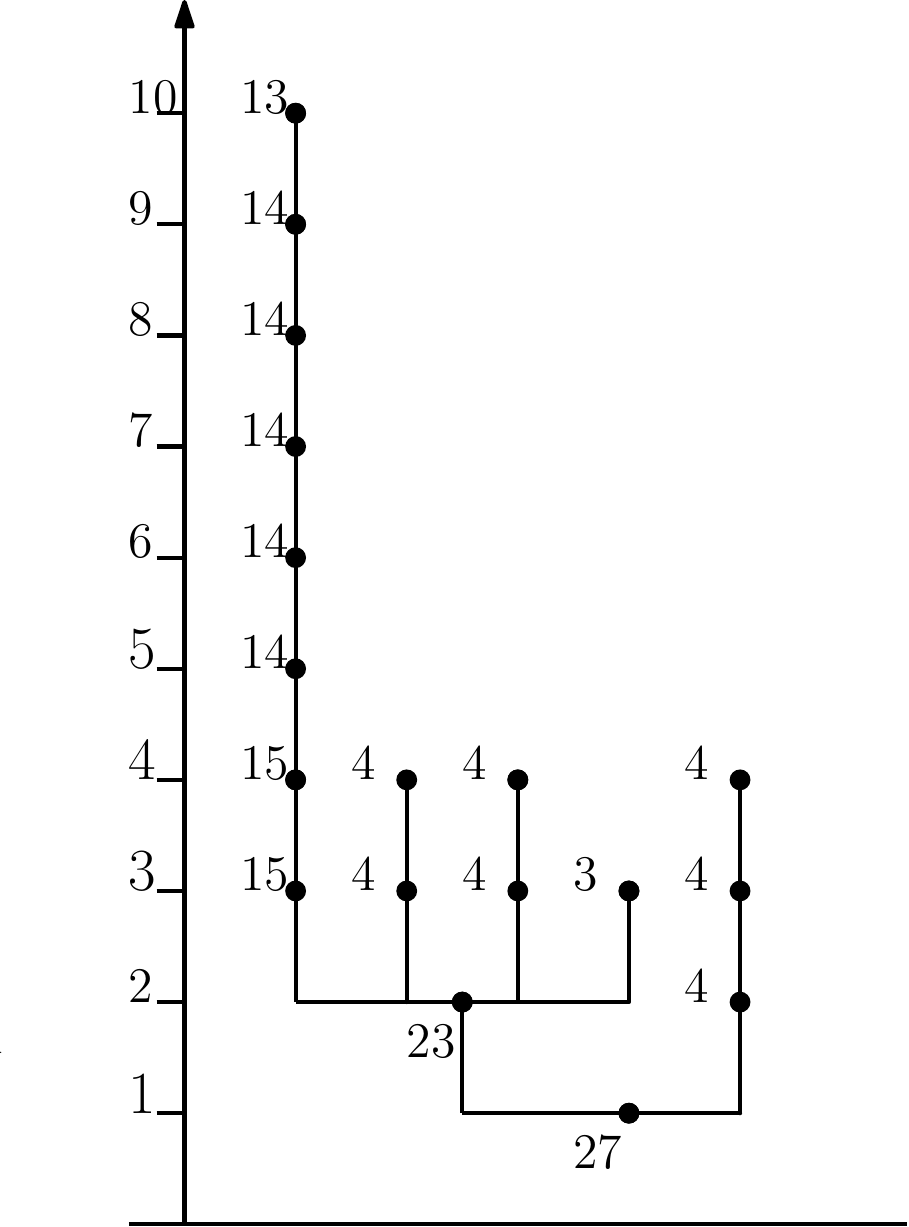}\\
		\includegraphics[height=1.3in]{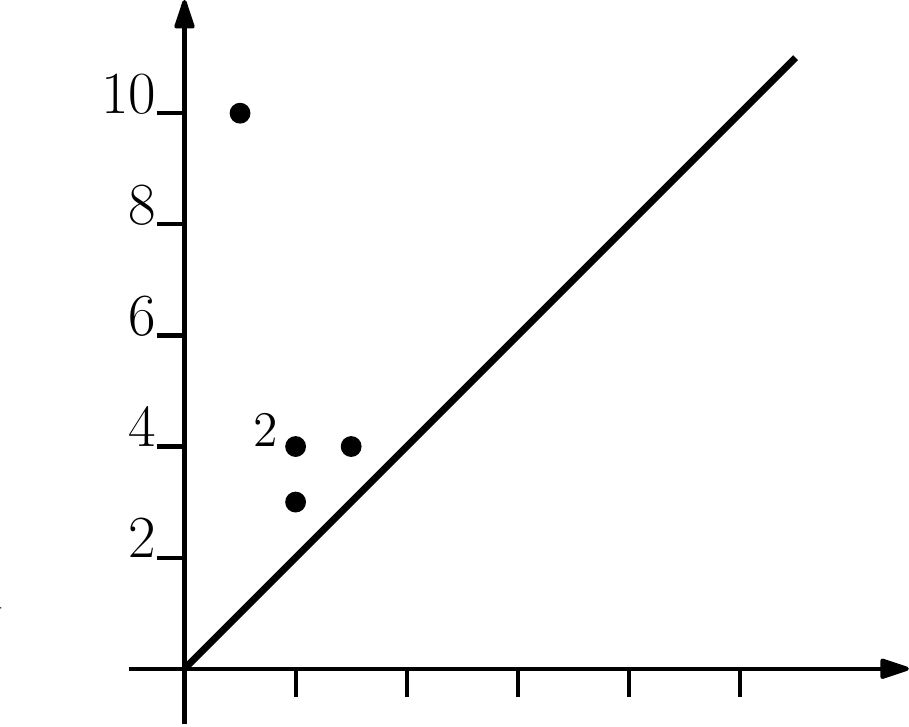}
        \caption{Zebra network}
    \end{subfigure}
\caption{Examples of community trees and persistent diagrams in real networks.
We consider three famous networks: the dolphin network, Zachary's karate club network, and Zebra network.
The top row shows their community trees
and the bottom row shows their persistent diagrams. 
The number attached to each node of community trees indicates the number of vertices belong to that node (community). 
The number attached to each point in the persistent diagram indicates number of components with the corresponding birth and death time. 
%The left two panels
%present the community tree and persistent diagram of a dolphin social network.
%The right two panels
%present the community tree and persistent diagram of a karate club network (the Zachary's karate club).
%The number at each node of a community denotes the number of vertices belong to this community. 
%In the right panel, the number $2$ indicate that
%there are two components having the same birth time and death time. 
}
\label{fig::ex01}
\end{figure}

We apply community trees to real network datasets in Figure~\ref{fig::ex01}. 
The left panel displays the community tree and persistent diagram of a dolphin social network\footnote{\url{http://www-personal.umich.edu/~mejn/netdata/}}.
This dataset describes the social network of 62 dolphins (with 159 edges) in a community of Doubtful Sound, New Zealand \citep{burnham2011aic}.
The middle panel shows the result of the Zachary's karate club network\footnote{\url{http://vlado.fmf.uni-lj.si/pub/networks/data/ucinet/ucidata.htm\#zachary}} \citep{zachary1977information}. 
This network contains 34 vertices and 78 edges. 
The right panel presents the community tree and persistent diagram of a zebra network\footnote{\url{http://moreno.ss.uci.edu/data.html\#zebra}} \citep{sundaresan2007network} that involves 27 zebras (27 vertices) and 111 interactions (111 edges). 
%The two branches at height $2$ indicates that the Zebra network contains two connected components.

We also attach a number to each node of a community tree
to describe the size of that node/community (size: number of vertices belonging to that node).
In persistent diagram, we attach a number to a point when there are multiple components with the same birth and death time.
%We observe 5 components inside this network and each of them correspond to one point in the persistent diagram. 
%The community tree informs how communities from CPM evolve and merge with each other.

\section{Discussion and Future Work}

In this paper, we introduce the concept of community trees
and persistent diagrams of networks. 
To study the stability of community trees,
we use the bottleneck distance of the corresponding
persistent diagrams. 
We prove that the bottleneck distance 
is upper bounded by a quantity called $TSN$
that can be evaluated without constructing community trees.
All these concepts are related to TDA, and, thus, this paper presents
a new way to apply TDA to network science.

Here, we comment on some possible future directions based on the current work.
\begin{itemize}
\item {\bf Practical algorithm for bounding the $TSN$.}
The $TSN$ is a powerful tool to 
control the change in a community tree without constructing the entire tree (and the corresponding persistent diagram). 
But as Theorem~\ref{thm::NP} has proved, computing the $TSN$ is an NP-complete problem. 
To use the $TSN$ to bound the change of networks,
we need a fast algorithm that provides a useful bound of the $TSN$. 
Finding such an algorithm will be left for future work.

\item {\bf Visualization tool using community tree.}
Community trees provide a nice and easy illustration
of the network.
Thus, it can be used as a visualization tool for a complex network. 
We plan to design visualization methods using community trees
and investigate the information loss during the visualization process. 

%However, the community tree itself is quiet plain so we can add other informations
%to it to make it a more comprehensive visualization tool.
%We plan to design a more powerful visualization tool using community trees in the future.

\item {\bf Effects from stochastic updates on community trees.}
In dynamic networks, networks change over time.
We may model the change of networks by a stochastic model where 
edges and nodes may be created or eliminated with certain probabilities. 
How the community trees will change under such stochastic model
will be an interesting research topic. 
Studies on this problem will allow us to understand the stability of
communities generated by the CPM when the network is dynamic.

\item {\bf Connections to overlapping communities.}
The community tree presents a new way to characterize overlapping communities.
The CPM was used to detect overlapping communities by 
the fact that
the same vertices may appear in different communities \cite{understanding_Palla_2005,porter2009communities}.
Using a community tree,
we can define the overlap between two communities
by the merging between their corresponding components.
%by the intersection of the tree components containing each of the communities.
We will study how different notions of overlaps 
are related to each other.

%\item {\bf Application of community trees in real networks.}
%We only consider a theoretical framework that 
%associates topological data analysis and network sciences. 
%What structure of a realistic network can a community tree discovers
%is still an open question. 
%We will analyze the community trees in real networks 

\end{itemize}

%\subsubsection*{Acknowledgments}

%Use unnumbered third level headings for the acknowledgments. All
%acknowledgments go at the end of the paper. Do not include
%acknowledgments in the anonymized submission, only in the final paper.

\appendix
\section{Proof of Theorem~\ref{thm::star}}

Before we prove Theorem~\ref{thm::star}, we first recall a property for two networks that differ only by one vertex:
\begin{lemma}
Let $G_1,G_2$ be two networks such that $G_2 =G_1\backslash\{v\}$, where $v\in V(G_1)$.
Then 
$$
d_B(T(G_1),T(G_2)) \leq 1. 
$$
\label{lem::one}
\end{lemma}
%Lemma~\ref{lem::one} is easy to see so we ignore its proof.

{\bf Proof of Theorem~\ref{thm::star}.}
Let $G$ and $G'$ be the two graphs being considered. 
We define $G_1$ to be the graph with vertex $V(G_1) = V(G)\cap V(G')$
and edges $E(G_1) = E(G)\cap E(G')$. 
%It is easy to see that 
%\begin{align*}
%RSN(G,G_1) = RSN(G,G'),
%\quad ASN(G_1,G') = ASN(G,G')
%\end{align*}
%so 
%\begin{equation}
%TSN(G,G') %= RSN(G,G') + ASN(G,G') 
%= RSN(G,G_1) + ASN(G_1,G').
%\label{eq::G1}
%\end{equation}

By triangle inequality, 
\begin{equation}
d_B(T(G),T(G')) \leq d_B(T(G),T(G_1)) +d_B(T(G_1),T(G')). 
\label{eq::tri}
\end{equation}
We will derive bounds on both quantities. 

{\bf Part I: Bounding $d_B(T(G),T(G_1)) $.}
Let $V_{RSN}$ denotes a possible candidate $V_0$ in Definition \ref{def::sn} that leads to $RSN(G,G')$. 
Namely, 
$$
\forall e\in E(G)\backslash E(G'), \,\,\nu(e) \cap V_{RSN} \neq \emptyset
$$
and $\|V_{RSN}\| = RSN(G,G')$, where $\|V\|$ denotes the cardinality of a vertex set $V$. 
Note that it is easy to see that $RSN(G,G') = RSN(G,G_1)$.

Define $G_2 = G\backslash V_{RSN}$ be the induced subgraph of $G$ by removing every vertex 
in $V_{RSN}$. 
Because $RSN(G,G') = RSN(G,G_1)$, 
$G_2\subset G_1\subset G$. 
This further implies
\begin{equation}
d_B(T(G),T(G_1))\leq d_B(T(G),T(G_2)).
\label{eq::G12}
\end{equation}

Let $V_{RSN}= \{v_1,\cdots,v_{RSN(G,G')}\}$.
We define a sequence of subgraphs $G^*_0,\cdots,G^*_{RSN(G,G')}$ such that $G^*_{0} = G$
and
$$
G^*_{t} =  G^*_{t-1}\backslash \{v_t\}, %, \cdots G^*_{RSN(G,G')} =  G^*_{RSN(G,G')-1}\backslash \{v_{RSN(G,G')}\} = G_2.
$$
for $t=1,\cdots, RSN(G,G')$. 
Note that $G^*_{RSN(G,G')} = G_2$. 
Namely, the sequence $G^*_{0},\cdots,G^*_{RSN(G,G')}$
is a sequence of graphs from $G$ to $G_2$ that differ by only one vertex. 
Because $G^*_{t} $ and $G^*_{t+1}$ differ by only one vertex, by Lemma~\ref{lem::one} and triangle inequalities
$$
d_B(T(G),T(G_2)) \leq \sum_{t=1}^{RSN(G,G')}d_B(T(G^*_{t-1},G^*_t)) \leq\sum_{t=1}^{RSN(G,G')} 1 = RSN(G,G').
$$
Combining the above inequality and equation \eqref{eq::G12}, we conclude
\begin{equation}
d_B(T(G),T(G_1))\leq RSN(G,G').
\label{eq::G1bound}
\end{equation}

{\bf Part II: Bounding $d_B(T(G_1),T(G')) $.}
By Definition \ref{def::sn}, $RSN(G_a,G_b) = ASN(G_b,G_a)$ for any two graphs $G_a$ and $G_b$.
Thus, %$ASN(G,G') = ASN(G_1,G') = RSN(G',G_1)$.
$ASN(G,G') = RSN(G',G)$.

Replacing $G$ by $G'$ in {\bf Part I}, we conclude
$$
d_B(T(G'),T(G_1))\leq RSN(G',G) = ASN(G,G').
$$
Using the above inequality and equations \eqref{eq::G1bound} and \eqref{eq::tri}, we obtain
\begin{align*}
d_B(T(G),T(G')) &\leq d_B(T(G),T(G_1)) +d_B(T(G_1),T(G'))\\ 
&\leq RSN(G,G')+ASN(G,G') \\
&= TSN(G,G').
\end{align*}
$\Box$

\newpage

%\section{References}

\bibliography{supplementary}

\begin{thebibliography}{22}
\providecommand{\natexlab}[1]{#1}
\providecommand{\url}[1]{\texttt{#1}}
\expandafter\ifx\csname urlstyle\endcsname\relax
  \providecommand{\doi}[1]{doi: #1}\else
  \providecommand{\doi}{doi: \begingroup \urlstyle{rm}\Url}\fi

\bibitem[Adamcsek et~al.(2006)Adamcsek, Palla, Farkas, Der{\'e}nyi, and
  Vicsek]{adamcsek2006cfinder}
B.~Adamcsek, G.~Palla, I.~J. Farkas, I.~Der{\'e}nyi, and T.~Vicsek.
\newblock Cfinder: locating cliques and overlapping modules in biological
  networks.
\newblock \emph{Bioinformatics}, 22\penalty0 (8):\penalty0 1021--1023, 2006.

\bibitem[Balakrishnan et~al.(2013)Balakrishnan, Narayanan, Rinaldo, Singh, and
  Wasserman]{balakrishnan2013cluster}
S.~Balakrishnan, S.~Narayanan, A.~Rinaldo, A.~Singh, and L.~Wasserman.
\newblock Cluster trees on manifolds.
\newblock In \emph{Advances in Neural Information Processing Systems}, pages
  2679--2687, 2013.

\bibitem[Burnham et~al.(2011)Burnham, Anderson, and Huyvaert]{burnham2011aic}
K.~P. Burnham, D.~R. Anderson, and K.~P. Huyvaert.
\newblock Aic model selection and multimodel inference in behavioral ecology:
  some background, observations, and comparisons.
\newblock \emph{Behavioral Ecology and Sociobiology}, 65\penalty0 (1):\penalty0
  23--35, 2011.

\bibitem[Chaudhuri and Dasgupta(2010)]{chaudhuri2010rates}
K.~Chaudhuri and S.~Dasgupta.
\newblock Rates of convergence for the cluster tree.
\newblock In \emph{Advances in Neural Information Processing Systems}, pages
  343--351, 2010.

\bibitem[Chaudhuri et~al.(2014)Chaudhuri, Dasgupta, Kpotufe, and von
  Luxburg]{chaudhuri2014consistent}
K.~Chaudhuri, S.~Dasgupta, S.~Kpotufe, and U.~von Luxburg.
\newblock Consistent procedures for cluster tree estimation and pruning.
\newblock \emph{IEEE Transactions on Information Theory}, 60\penalty0
  (12):\penalty0 7900--7912, 2014.

\bibitem[Chen(2016)]{chen2016generalized}
Y.-C. Chen.
\newblock Generalized cluster trees and singular measures.
\newblock \emph{arXiv preprint arXiv:1611.02762}, 2016.

\bibitem[Cohen-Steiner et~al.(2007)Cohen-Steiner, Edelsbrunner, and
  Harer]{cohen2007stability}
D.~Cohen-Steiner, H.~Edelsbrunner, and J.~Harer.
\newblock Stability of persistence diagrams.
\newblock \emph{Discrete \& Computational Geometry}, 37\penalty0 (1):\penalty0
  103--120, 2007.

\bibitem[DINUR and SAFRA(2005)]{dinur2005hardness}
I.~DINUR and S.~SAFRA.
\newblock On the hardness of approximating minimum vertex cover.
\newblock \emph{Annals of mathematics}, 162\penalty0 (1):\penalty0 439--485,
  2005.

\bibitem[Duan et~al.(2012)Duan, Li, Li, and Lu]{duan2012incremental}
D.~Duan, Y.~Li, R.~Li, and Z.~Lu.
\newblock Incremental k-clique clustering in dynamic social networks.
\newblock \emph{Artificial Intelligence Review}, pages 1--19, 2012.

\bibitem[Eldridge et~al.(2015)Eldridge, Belkin, and Wang]{eldridge2015beyond}
J.~Eldridge, M.~Belkin, and Y.~Wang.
\newblock Beyond hartigan consistency: Merge distortion metric for hierarchical
  clustering.
\newblock In \emph{Conference on Learning Theory}, pages 588--606, 2015.

\bibitem[Fortunato(2010)]{fortunato2010community}
S.~Fortunato.
\newblock Community detection in graphs.
\newblock \emph{Physics reports}, 486\penalty0 (3):\penalty0 75--174, 2010.

\bibitem[Jisu et~al.(2016)Jisu, Chen, Balakrishnan, Rinaldo, and
  Wasserman]{jisu2016statistical}
K.~Jisu, Y.-C. Chen, S.~Balakrishnan, A.~Rinaldo, and L.~Wasserman.
\newblock Statistical inference for cluster trees.
\newblock In \emph{Advances In Neural Information Processing Systems}, pages
  1839--1847, 2016.

\bibitem[Palla et~al.(2005)Palla, Der\'{e}nyi, Farkas1, and
  Vicsek]{understanding_Palla_2005}
G.~Palla, I.~Der\'{e}nyi, I.~Farkas1, and T.~Vicsek.
\newblock Uncovering the overlapping community structure of complex networks in
  nature and society.
\newblock \emph{Nature Letters}, 435:\penalty0 814--818, 2005.

\bibitem[Palla et~al.(2007)Palla, Barab{\'a}si, and
  Vicsek]{palla2007quantifying}
G.~Palla, A.~Barab{\'a}si, and T.~Vicsek.
\newblock Quantifying social group evolution.
\newblock \emph{Nature}, 446\penalty0 (7136):\penalty0 664--667, 2007.

\bibitem[Pollner et~al.(2005)Pollner, Palla, and
  Vicsek]{pollner2005preferential}
P.~Pollner, G.~Palla, and T.~Vicsek.
\newblock Preferential attachment of communities: The same principle, but a
  higher level.
\newblock \emph{EPL (Europhysics Letters)}, 73\penalty0 (3):\penalty0 478,
  2005.

\bibitem[Porter et~al.(2009)Porter, Onnela, and Mucha]{porter2009communities}
M.~A. Porter, J.-P. Onnela, and P.~J. Mucha.
\newblock Communities in networks.
\newblock \emph{Notices of the AMS}, 56\penalty0 (9):\penalty0 1082--1097,
  2009.

\bibitem[Stuetzle(2003)]{stuetzle2003estimating}
W.~Stuetzle.
\newblock Estimating the cluster tree of a density by analyzing the minimal
  spanning tree of a sample.
\newblock \emph{Journal of classification}, 20\penalty0 (1):\penalty0 025--047,
  2003.

\bibitem[Sundaresan et~al.(2007)Sundaresan, Fischhoff, Dushoff, and
  Rubenstein]{sundaresan2007network}
S.~R. Sundaresan, I.~R. Fischhoff, J.~Dushoff, and D.~I. Rubenstein.
\newblock Network metrics reveal differences in social organization between two
  fission--fusion species, grevy’s zebra and onager.
\newblock \emph{Oecologia}, 151\penalty0 (1):\penalty0 140--149, 2007.

\bibitem[Wasserman(2016)]{wasserman2016topological}
L.~Wasserman.
\newblock Topological data analysis.
\newblock \emph{arXiv preprint arXiv:1609.08227}, 2016.

\bibitem[Yannakakis and Gavril(1980)]{yannakakis1980edge}
M.~Yannakakis and F.~Gavril.
\newblock Edge dominating sets in graphs.
\newblock \emph{SIAM Journal on Applied Mathematics}, 38\penalty0 (3):\penalty0
  364--372, 1980.

\bibitem[Zachary(1977)]{zachary1977information}
W.~W. Zachary.
\newblock An information flow model for conflict and fission in small groups.
\newblock \emph{Journal of anthropological research}, 33\penalty0 (4):\penalty0
  452--473, 1977.

\bibitem[Zhang et~al.(2006)Zhang, Ning, and Zhang]{zhang2006identification}
S.~Zhang, X.~Ning, and X.-S. Zhang.
\newblock Identification of functional modules in a ppi network by clique
  percolation clustering.
\newblock \emph{Computational biology and chemistry}, 30\penalty0 (6):\penalty0
  445--451, 2006.

\end{thebibliography}
\bibliographystyle{abbrvnat}

\end{document}